\documentclass[final,3p,times]{elsarticle}

\usepackage{graphicx}

\usepackage{amssymb}
\usepackage{amsthm}

\usepackage{amsmath,mathrsfs}

\newtheorem{Def}{Definition}[section]
\newtheorem{Eg}{Example}[section]
\newtheorem{Prop}{Proposition}[section]
\newtheorem{Thm}{Theorem}[section]
\newtheorem{Lem}{Lemma}[section]
\newtheorem{Coro}{Corollary}[section]

\newproof{Pot2}{Proof of Theorem \ref{Thomo}}

\newcommand{\ms}{\mathscr}

\newcommand{\ul}{\underline}
\newcommand{\ol}{\overline}

\newcommand{\HH}{\mathcal{H}}
\newcommand{\G}{\mathcal{G}}

\newcommand{\A}{\mathcal{A}}

\newcommand{\la}{\langle}
\newcommand{\ra}{\rangle}
\newcommand{\ha}{\hookrightarrow}

\makeatletter
  \newcommand\figcaption{\def\@captype{figure}\caption}
  \newcommand\tabcaption{\def\@captype{table}\caption}
\makeatother

\begin{document}

\begin{frontmatter}


\title{Information-theoretic measures associated with rough set approximations}


\author[pz1,pz2]{Ping Zhu}
\ead{pzhubupt@gmail.com}


\author[pz2]{Qiaoyan Wen}
\ead{wqy@bupt.edu.cn}

\address[pz1]{School of Science, Beijing University of Posts and Telecommunications, Beijing 100876,
China}

\address[pz2]{State Key Laboratory of Networking and Switching, Beijing University of Posts and Telecommunications, Beijing 100876, China}

\begin{abstract}
Although some information-theoretic measures of uncertainty or granularity have been proposed in rough set theory, these measures are only dependent on the underlying partition and the cardinality of the universe, independent of the lower and upper approximations. It seems somewhat unreasonable since the basic idea of rough set theory aims at describing vague concepts by the lower and upper approximations. In this paper, we thus define new information-theoretic entropy and co-entropy functions associated to the partition and the approximations to measure the uncertainty and granularity of an approximation space. After introducing the novel notions of entropy and co-entropy, we then examine their properties. In particular, we discuss the relationship of co-entropies between different universes. The theoretical development is accompanied by illustrative numerical examples.
\end{abstract}

\end{frontmatter}

\section{Introduction}
\label{introd} To handle inexact, uncertain or vague knowledge in some
information systems, Pawlak developed rough set theory in the early
1980s \cite{Paw82,Paw91}. Since then we have witnessed a systematic,
world-wide growth of interest in rough set theory and its applications in a number of fields, such as granular computing,
data mining, decision analysis, pattern recognition, and approximate reasoning
\cite{LC97,PS98a,PS98b,ZYO03,ZW06,Zia94}.

The starting point of rough set theory in \cite{Paw82,Paw91} is the idea that elements of a
universe having the same description are indiscernible with respect to the available information. The indiscernibility was
described by an equivalence relation in the way that two elements are related by the relation if and only if they are indiscernible from
each other. As is well known, any equivalence relation defined on a universe $U$ determines a partition of $U$ into a collection of equivalence classes (blocks): each class contains all and only the elements that are mutually equivalent among them. Any partition $\pi$ of $U$ represents a piece of knowledge about the
elements of $U$ forming a classification and so any equivalence class induced by $\pi$ is interpreted as a granule of knowledge contained in
(or supported by) $\pi$.

According to Pawlak's terminology expressed in \cite{pawlak1992rough}, any subset $X$ of the universe $U$ is called
a concept in $U$. If the concept $X$ is a union of equivalence classes from $\pi$, then $X$ is precise in $\pi$, otherwise $X$ is vague. The basic idea of rough set theory consists in replacing vague concepts with a pair of precise concepts, its lower and upper approximations \cite{pawlak1992rough}, and thus, a basic problem in this framework is to reason about the accessible granules of knowledge. To this end, various knowledge granulations (also, information granulations or granulation measures), as an average measure of knowledge granules, have been proposed and addressed in \cite{BPA98,BCC07,LSLW06,LS04,miao1998relationships,sen2009generalized,Wie99,xu2009knowledge,Yao03b,Yao03,YZ09p,Zhu09improved}. Among them, there are several information-theoretic measures of uncertainty or granularity for rough sets \cite{BPA98,BCC07,LSLW06,liang2008information,LS04,miao1998relationships,sen2009generalized,Wie99,Yao03b}, which are based upon the important notion of entropy introduced by Shannon \cite{Sha48}; for more details, we refer the reader to the excellent survey papers \cite{bianucci2009information,yao2010notes}.

It is worth noting that the information-theoretic measures mentioned above are only dependent on the sizes of equivalence classes (essentially, the underlying partition) and the cardinality of the universe, independent of the lower and upper approximation operators. For example, in \cite{DG98,miao1998relationships,Wie99,Yao03} the information entropy $H(\pi)$ of the partition $\pi=\{U_1,U_2,\ldots,U_k\}$ is defined as
\begin{equation*}
H(\pi)=-\sum_{i=1}^k\frac{n_i}{n}\log\frac{n_i}{n},
\end{equation*} where $n_i$ is the cardinality of $U_i$ and $n=\sum_{i=1}^k n_i$. As a result, it often yields that some partitions like $\{\{1\},\{2\}\}$ and $\{\{1,2\},\{3,4\}\}$ have the same entropy (or co-entropy). This seems somewhat unreasonable since the basic idea of rough set theory aims at describing vague concepts by the lower and upper approximations. In other words, the result of this description relies on both the partition and the approximations. In light of this, we should pay more attention to the lower and upper approximation operators.

The previous observation motivates us to propose another information-theoretic entropy function to measure the uncertainty associated to the partition and the approximation operators in this paper.
More concretely, given a universe $U$ with $n$ elements and a partition $\pi$ of $U$, we take count of the subsets of $U$ described by every pair of lower and upper approximations. Assume that $r_i$, $1\leq i\leq m$, is the number of subsets described by the rough set approximation $(A_i,A'_i)$ and every subset of $U$ appears with the same probability. It follows that the rough set approximation $(A_i,A'_i)$ appears with the accumulative probability $r_i/2^n$ since the amount of all subsets of $U$ is precisely $2^n$. In this way, we obtain a probability distribution
\begin{equation*}
P(\pi)=\left(\frac{r_1}{2^n},\frac{r_2}{2^n},\ldots,\frac{r_m}{2^n}\right).
\end{equation*}
It gives rise to an information entropy, say $\HH(\pi)$, according to Shannon's information theory \cite{Sha48}. On the other hand, we can get by the probability distribution a co-entropy $\G(\pi)$. It turns out that $\HH(\pi)+\G(\pi)=n$. After exploring some properties of the entropy and co-entropy, we discuss the relationships of co-entropies between different universes. Roughly speaking, the co-entropy monotonically increases when the partition becomes coarser. For example, the co-entropy of $\{\{1,2\},\{3,4\}\}$ is greater than that of $\{\{1\},\{2\}\}$.

The remainder of the paper is structured as follows. In Section 2,
we briefly review some basics of Pawlak's rough set theory and the
information-theoretic measures of uncertainty and granularity for rough sets in the literature. Section 3 is devoted to our novel notions of entropy and co-entropy and their properties. We address the relationship of co-entropies between different universes in Section 4 and conclude the paper in
Section 5 with a brief discussion on the future research.

\section{Preliminaries}
\label{prelim} This section consists of two subsections. We briefly recall the definition of Pawlak's rough sets in the first subsection and then review two information-theoretic measures of uncertainty and granularity in rough set theory in the second subsection.

\subsection{Rough sets}
We start by recalling some basic notions in Pawlak's rough set
theory \cite{Paw82,Paw91}.

Let $U$ be a finite and nonempty universal set, and let $R\subseteq
U\times U$ be an equivalence relation on $U$. Denote by $U/R$ the
set of all equivalence classes induced by $R$. Such equivalence
classes are also called {\it elementary sets}; every union (not
necessarily nonempty) of elementary sets is called a {\it definable
set}.

For any $X\subseteq U$, one can characterize $X$ by a pair of
lower and upper approximations. The {\it lower approximation}
$\ul{app}_RX$ of $X$ is defined as the greatest definable set contained
in $X$, while the {\it upper approximation} $\ol{app}_RX$ of $X$ is
defined as the least definable set containing $X$. Formally,
$$\ul{app}_RX= \cup\{C\in U/R\;|\,C\subseteq X\}\ \mbox{ and }\
\ol{app}_RX= \cup\{C\in U/R\;|\,C\cap X\neq\emptyset\}.$$ The pair $\left(\ul{app}_RX,\ol{app}_RX\right)$ is referred to as the {\it rough set approximation} of $X$. It follows
immediately from definition that $\ul{app}_RX\subseteq X\subseteq\ol{app}_RX$
for any $X\subseteq U$.

The ordered pair $\langle U,R\rangle$
is said to be an {\it approximation space}. A {\it rough set} in $\langle U,R\rangle$ is the family of all subsets of $U$ having the same lower and upper approximations. Thus, the general notion of rough set can be simply
identified with the rough approximation of any given set.

Recall that a {\it partition} of $U$ is a collection of nonempty
subsets of $U$ such that every element $x$ of $U$ is in exactly one
of these subsets; such subsets making up the partition are called {\it blocks}. We write $\Pi(U)$ for the set of all partitions of
$U$ and $\ms{P}(U)$ for the power set of $U$. It is well-known that
the notions of partition and equivalence relation are essentially
equivalent, that is, for any equivalence relation $R$ on $U$, the
set $U/R$ is a partition of $U$, and conversely, from any partition
$\pi$ of $U$, one can define an equivalence relation $R_\pi$ on $U$
such that $U/R_\pi=\pi$ in the obvious way. Thus, we sometimes say
that the ordered pair $\langle U,\pi\rangle$ is an approximation
space and write $\ul{app}_\pi X$ and $\ol{app}_\pi X$ for $\ul{app}_{R_\pi}X$ and
$\ol{app}_{R_\pi}X$, respectively. More generally, we will use
equivalence relation and partition indiscriminately.

If a universe $U$ has more than one element, it is always possible to
introduce at least two canonical partitions: One is the trivial
partition, denoted by $\check{\pi}$, consisting of a unique
block, and the other is the discrete partition, denoted
by $\hat{\pi}$, consisting of all singletons from $U$. Formally,
$$\check{\pi} =\{U\}\mbox{ and }\hat{\pi} =\{\{x\}\;|\,x\in U\}.$$

We now define a partial order ``$\preceq$" on $\Pi(U)$: For any
$\pi,\sigma\in\Pi(U)$, $\sigma\preceq\pi$ if and only if for any $C\in\sigma$,
there exists $D\in\pi$ such that $C\subseteq D$.
For instance, $\hat{\pi}\preceq\pi\preceq\check{\pi}$ for any
$\pi\in\Pi(U)$. We say that $\sigma$ is {\it finer} than $\pi$ and
that $\pi$ is {\it coarser} than $\sigma$ if $\sigma\preceq\pi$.
When $\sigma\prec\pi$, that is, $\sigma\preceq\pi$ and
$\sigma\neq\pi$, we say that $\sigma$ is {\it strictly finer} than
$\pi$ and that $\pi$ is {\it strictly coarser} than $\sigma$.
Informally, this means that $\sigma$ is a further fragmentation of
$\pi$.

\subsection{Information-theoretic measures}
In this subsection, we review two information-theoretic measures associated with rough sets in the literature. These measures are concerned with the uncertainty or granularity of knowledge provided by a partition.

In \cite{DG98,miao1998relationships,Wie99,Yao03}, Shannon entropy \cite{Sha48} has been used as a measure of information for rough set theory as follows. For subsequent need, we fix a notational convention: Throughout the paper, all logarithms are to base $2$ unless otherwise specified.
\begin{Def}[\cite{DG98,miao1998relationships,Wie99,Yao03}]\label{DgranmeaG}
Let $\langle U,\pi\rangle$ be an approximation space, where the partition $\pi$ consists of blocks $U_i$, $1\leq i\leq k$, each having cardinality $n_i$. The {\rm information entropy} $H(\pi)$ of partition $\pi$ is defined by
\begin{equation}\label{Eentropy}
H(\pi)=-\sum_{i=1}^k\frac{n_i}{n}\log\frac{n_i}{n}, \mbox{ where }n=\sum_{i=1}^k n_i.
\end{equation}
\end{Def}

When $\pi=\check{\pi}$, the entropy function $H$ achieves the minimum value $0$, and when $\pi=\hat{\pi}$, it achieves the maximum value $\log n$. Moreover, it has been shown in \cite{Wie99} that for any two partitions $\pi$ and $\sigma$ of $U$, if $\sigma\prec\pi$, then $H(\sigma)>H(\pi)$.

The equation (\ref{Eentropy}) can be rewritten as follows:
\begin{equation}\label{Eentropy'}
H(\pi)=\log n-\sum_{i=1}^k\frac{n_i}{n}\log {n_i}.
\end{equation}

Recall that the Hartley measure \cite{hartley1928} of uncertainty for a finite set $X$ is
$$H(X)=\log |X|,$$ where ``$|X|$" denotes the cardinality of the set
$X$. It measures the amount of uncertainty associated with a finite set of possible
alternatives, the nonspecificity inherent in the set.

The first term $\log n$ (i.e., $\log |U|$)
in Eq. (\ref{Eentropy'}) is exactly the Hartley measure of $U$, which is a constant independent
of any partition. The second term of the equation is basically an expectation of granularity with
respect to all blocks in a partition.
This quantity has been used by Yao to measure the granularity of a partition in \cite{Yao03} and has been defined by Liang and Shi as the rough entropy of knowledge in an approximation space in \cite{LS04}. This quantity has also been referred to as co-entropy by some scholars (see, for example, \cite{bianucci2009information,BCC07}).

\begin{Def}[\cite{bianucci2009information,BCC07,LS04,Yao03}]\label{DgranmeaRG}
Let $\langle U,\pi\rangle$ be an approximation space, where the partition $\pi$ consists of blocks $U_i$, $1\leq i\leq k$, each having cardinality $n_i$. The {\rm co-entropy} $G(\pi)$ of partition $\pi$ is defined by
\begin{equation}\label{Eroughent}
G(\pi)=\sum_{i=1}^k\frac{n_i}{n}\log n_i, \mbox{ where }n=\sum_{i=1}^k n_i.
\end{equation}
\end{Def}

It follows immediately from definition that $$H(\pi)+G(\pi)=\log n.$$ Contrary to the uncertainty measure $H$, the co-entropy function $G$ achieves the maximum value $\log n$ when $\pi=\check{\pi}$ and the minimum value $0$ when $\pi=\hat{\pi}$; moreover, it has been known \cite{LS04} that for any two partitions $\pi$ and $\sigma$ of $U$, if $\sigma\prec\pi$, then $G(\sigma)<G(\pi)$.

As argued in \cite{bianucci2009information,BCC07}, the entropy $H(\pi)$ can be interpreted as the uncertainty measure of the partition $\pi$, while the  co-entropy $G(\pi)$ can be regarded as the granularity measure of $\pi$.  In \cite{sen2009generalized}, Sen and Pal introduced two other entropy measures for crisp sets and fuzzy sets with (crisp or fuzzy) equivalence relations or (crisp or fuzzy) tolerance relations, which are based upon the roughness measures of $X$ and of the complement of $X$ in the universe and have been used to analyze the grayness and spatial ambiguities in images. Under
the same name, there are some different concepts of entropy in the literature of rough set theory (see, for example, \cite{liang2002new,QL08}).

\section{A novel pair of entropy and co-entropy}
In this section, we first introduce a novel entropy and the corresponding co-entropy and then explore their properties.

Let us begin with some notations. Throughout this section, we write $\langle U,\pi\rangle$ for an approximation space and assume that $|U|=n$.
Given a $\langle U,\pi\rangle$, we use $\mathcal{A}(U,\pi)$ to denote the set of rough set approximations of all subsets of $U$. More formally, we set
\begin{equation}\label{Ersa}
\mathcal{A}(U,\pi)=\left\{\left(\ul{app}_\pi X, \ol{app}_\pi X\right)\,\Bigl|\,X\subseteq U\right\}.
\end{equation}

It follows from Eq. (\ref{Ersa}) that $\mathcal{A}(U,\pi)$ has at least two elements: $(\emptyset,\emptyset)$ and $(U,U)$. If $n=1$, then $\mathcal{A}(U,\pi)$ exactly consists of the two elements; if $n>1$ and $\pi=\check{\pi}$, then $\mathcal{A}(U,\pi)$ contains one more element $(\emptyset,U)$; for any $n\geq1$, if $\pi=\hat{\pi}$, then we see that $\mathcal{A}(U,\pi)=\left\{\left(X, X\right)\,|\,X\subseteq U\right\}$, which consists of $2^{n}$ elements. Note that the set $\mathcal{A}(U,\pi)$ is not a multiset, that is, the same element cannot appear more than once in $\mathcal{A}(U,\pi)$. In general, we have that $|\mathcal{A}(U,\pi)|\leq2^n$ since the subset $X$ of $U$ in Eq. (\ref{Ersa}) has only $2^n$ alternatives.

For simplicity, we use $m$ to stand for $|\mathcal{A}(U,\pi)|$. For any $(A_i,A'_i)\in\mathcal{A}(U,\pi)$, $1\leq i\leq m$, we set
\begin{equation}\label{Enumber}
\A_i=\left\{X\subseteq U\,\Bigl|\,\left(\ul{app}_\pi X, \ol{app}_\pi X\right)=(A_i,A'_i)\right\}\mbox{ and } |\A_i|=r_i.
\end{equation}
In other words, $r_i$ is the number of subsets of $U$ that have the rough set approximation $(A_i,A'_i)$. It turns out that $\{\A_1,\A_2,\ldots,\A_m\}$ gives rise to a partition of $\ms{P}(U)$. Therefore, we get by Eq. (\ref{Ersa}) that
$$\sum_{i=1}^m r_i=2^n.$$


To illustrate the above concepts, let us see an example.
\begin{Eg}\label{Enovel} Consider $U=\{1,2,3,4\}$ and $\pi=\{\{1,2\},\{3,4\}\}$. In this case, $U$ has $16$ subsets. For each subset $X$ of $U$, we compute the rough set approximation of $X$; the results are listed in Table \ref{Tnovel}.
\begin{table}[htb]\centering\tabcaption{The subsets and corresponding rough set approximations in Example \ref{Enovel}.}\vspace{0.2cm}
\begin{tabular}{|cc|c c|cc|c c|}\hline
subset & approximation & subset & approximation &subset & approximation &subset & approximation  \\
\hline
$\emptyset$ &$\left(\emptyset,\emptyset\right)$ &$\{1\}$&$\left(\emptyset,\{1,2\}\right)$& $\{2\}$&$\left(\emptyset,\{1,2\}\right)$ &$\{3\}$&$\left(\emptyset,\{3,4\}\right)$\\ \hline
$\{4\}$ &$\left(\emptyset,\{3,4\}\right)$ &$\{1,2\}$&$\left(\{1,2\},\{1,2\}\right)$& $\{1,3\}$&$\left(\emptyset,U\right)$ &$\{1,4\}$&$\left(\emptyset,U\right)$\\ \hline
$\{2,3\}$ &$\left(\emptyset,U\right)$ &$\{2,4\}$&$\left(\emptyset,U\right)$& $\{3,4\}$&$\left(\{3,4\},\{3,4\}\right)$ &$\{1,2,3\}$&$\left(\{1,2\},U\right)$\\ \hline
$\{1,2,4\}$ &$\left(\{1,2\},U\right)$ &$\{1,3,4\}$&$\left(\{3,4\},U\right)$& $\{2,3,4\}$&$\left(\{3,4\},U\right)$ &$U$&$\left(U,U\right)$\\\hline
\end{tabular}
\label{Tnovel}
\end{table}

Hence, we see that
$$\mathcal{A}(U,\pi)=\left\{\left(\emptyset,\emptyset\right),\left(\emptyset,\{1,2\}\right),\left(\emptyset,\{3,4\}\right),
\left(\{1,2\},\{1,2\}\right),\left(\emptyset,U\right),\left(\{3,4\},\{3,4\}\right),\left(\{1,2\},U\right),\left(\{3,4\},U\right), \left(U,U\right)\right\}.$$ As an example, let us calculate $r_2$. By definition,
\begin{eqnarray*}
  r_2 &=& \left|\left\{X\subseteq U\,\Bigl|\,\left(\ul{app}_\pi X, \ol{app}_\pi X\right)=\left(\emptyset,\{1,2\}\right)\right\}\right| \\
   &=&\left|\left\{\{1\},\{2\} \right\}\right|\\
   &=&2.
\end{eqnarray*}This is exactly the number of subsets of $U$ that have the rough set approximation $\left(\emptyset,\{1,2\}\right)$, which can be counted from the table. In light of this, we may get Table \ref{Tnovel2} by  rearranging Table \ref{Tnovel}. It follows immediately from Table \ref{Tnovel2} that $r_1=r_4=r_6=r_9=1$, $r_2=r_3=r_7=r_8=2$, and $r_5=4$.
\begin{table}[htb]\centering\tabcaption{The rough set approximations and corresponding subsets in Example \ref{Enovel}.}\vspace{0.2cm}
\begin{tabular}{|cc|c c|cc|}\hline
approximation &subsets  & approximation & subsets & approximation &subsets \\
\hline
$\left(\emptyset,\emptyset\right)$ &$\emptyset$ &$\left(\emptyset,\{1,2\}\right)$&$\{1\}, \{2\}$ & $\left(\emptyset,\{3,4\}\right)$&$\{3\}, \{4\}$\\ \hline
 $\left(\{1,2\},\{1,2\}\right)$&$\{1,2\}$&$\left(\emptyset,U\right)$ & $\{1,3\},\{1,4\},\{2,3\},\{2,4\}$&$\left(\{3,4\},\{3,4\}\right)$ &$\{3,4\}$\\ \hline
 $\left(\{1,2\},U\right)$&$\{1,2,3\},\{1,2,4\}$&$\left(\{3,4\},U\right)$&$\{1,3,4\},\{2,3,4\}$&$\left(U,U\right)$& $U$\\\hline
\end{tabular}
\label{Tnovel2}
\end{table}
\end{Eg}

Because we are concerned with the partition granulation of $\langle U,\pi\rangle$ with respect to the
approximation operators $\ul{app}$ and $\ol{app}$, we may assume that every subset of $U$ appears with the same probability $1/2^n$. As a result, the rough set approximation $(A_i,A'_i)$ appears with the accumulative probability $r_i/2^n$ and we thus obtain a probability distribution
\begin{equation}\label{Eprobdist}
P(\pi)=\left(\frac{r_1}{2^n},\frac{r_2}{2^n},\ldots,\frac{r_m}{2^n}\right).
\end{equation}

According to Shannon's information theory \cite{Sha48}, the Shannon entropy function of the probability distribution $P(\pi)$ is defined as follows.
\begin{Def}\label{DShannonent}Keep the notations as above. The {\rm information entropy} $\HH(\pi)$ of $\langle U,\pi\rangle$ (with respect to the approximation operators $\ul{app}$ and $\ol{app}$) is defined by
\begin{equation}\label{EShannonent}
\HH(\pi)=\HH(P(\pi))=-\sum_{i=1}^m\frac{r_i}{2^n}\log\frac{r_i}{2^n}.
\end{equation}
\end{Def}

In the above definition, for simplicity we have used the notation $\HH(\pi)$ instead of $\HH(U,\pi)$. Following the explanation of Shannon entropy in information theory, the quantity $\HH(\pi)$ measures the uncertainty associated to the partition $\pi$ with respect to the approximation operators $\ul{app}$ and $\ol{app}$.
For instance, the probability distribution corresponding to the partition $\pi=\{\{1,2\},\{3,4\}\}$ in Example \ref{Enovel} is $$P(\pi)=\left(\frac{1}{2^4},\frac{2}{2^4},\frac{2}{2^4},\frac{1}{2^4},\frac{4}{2^4},\frac{1}{2^4},\frac{2}{2^4},
\frac{2}{2^4},\frac{1}{2^4}\right).$$
It follows from Definition \ref{DShannonent} that
\begin{eqnarray*}
\HH(\pi)&=& -\sum_{i=1}^9\frac{r_i}{2^4}\log\frac{r_i}{2^4} \\
 &=& -\Biggl[\frac{1}{2^4}\log\frac{1}{2^4}+\frac{2}{2^4}\log\frac{2}{2^4}+\frac{2}{2^4}\log\frac{2}{2^4}
 +\frac{1}{2^4}\log\frac{1}{2^4}+\frac{4}{2^4}\log\frac{4}{2^4}\\&&\quad+\frac{1}{2^4}\log\frac{1}{2^4}+\frac{2}{2^4}\log\frac{2}{2^4}+\frac{2}{2^4}\log\frac{2}{2^4}
 +\frac{1}{2^4}\log\frac{1}{2^4}\Biggr]\\
 &=&3.
\end{eqnarray*}

Similar to other entropy functions in rough set theory, the information entropy in Definition \ref{DShannonent} has the following properties.
\begin{Thm}\label{Pentfunc}\
\begin{enumerate}[(1)]
  \item For any $\pi,\sigma\in\Pi(U)$, if $\sigma\prec\pi$, then $\HH(\sigma)>\HH(\pi)$.
  \item The entropy function $\HH$ reaches the maximum value $n$ for the
finest partition $\hat{\pi}$.
  \item The entropy function $\HH$ reaches the minimum value $n-\frac{2^n-2}{2^n}\log(2^n-2)$ for the coarsest partition
$\check{\pi}$.
\end{enumerate}
\end{Thm}
\begin{proof}
(1) Without loss of generality, we may assume that $\pi=\{U_1,U_2,\ldots,U_k\}$ and $\sigma=\{U_a,U_b,U_2,\ldots,U_k\}$, where $U_a\cup U_b=U_1$. Suppose that $|\mathcal{A}(U,\pi)|=m$ and for any $(A_i,A'_i)\in\mathcal{A}(U,\pi)$, $1\leq i\leq m$, we write $r_i$ for
\begin{equation*}
\left|\left\{X\subseteq U\,\Bigl|\,\left(\ul{app}_\pi X, \ol{app}_\pi X\right)=(A_i,A'_i)\right\}\right|.
\end{equation*}
Based on the partition $\pi$, the power set $\ms{P}(U)$ is partitioned into $m$ blocks and the $i$-th block has the cardinality $r_i$. Similarly, we denote by $s_j$ the cardinality of the $j$-th block of $\ms{P}(U)$ associated to the partition $\sigma$. We now consider the elements of $\mathcal{A}(U,\sigma)$. For any $(B_j,B'_j)\in\mathcal{A}(U,\sigma)$, there are two possibilities: One is that $(B_j,B'_j)\in\mathcal{A}(U,\pi)$, say $(B_j,B'_j)=(A_{i_j},A'_{i_j})$ for some $i_j$. In this case, it is clear that $s_j=r_{i_j}$. The other case is that $(B_j,B'_j)\in\mathcal{A}(U,\sigma)\backslash\mathcal{A}(U,\pi)$, where the symbol $A\backslash B$ denotes the set of all elements which are members of $A$ but not members of $B$. It follows that for some $i_j$,
$$\left\{X\subseteq U\,\Bigl|\,\left(\ul{app}_\sigma X, \ol{app}_\sigma X\right)=(B_j,B'_j)\right\}\subsetneq\left\{X\subseteq U\,\Bigl|\,\left(\ul{app}_\pi X, \ol{app}_\pi X\right)=(A_{i_j},A'_{i_j})\right\},$$ because the partition $\sigma$ is strictly finer than $\pi$. In this case, we also see that the $i_j$-th block provided by $\pi$ is partitioned into smaller blocks and thus $r_{i_j}=\sum_j s_j>s_j$. In summary, we get that either $r_i=s_j$ or $r_i=\sum_j s_{i_j}>s_{i_j}$, and moreover, the latter case must exist as $\sigma\prec\pi$. We thus assume that $r_i=s_{i_j}$ for $i\in I_1$ and $r_i=\sum_j s_{i_j}>s_{i_j}$ for $i\in I_2$, where $I_2\neq\emptyset$ and $I_1\cup I_2=\{1,2\ldots,m\}$. Let us compare $\HH(\sigma)$ with $\HH(\pi)$.
\begin{eqnarray*}
\HH(\pi)&=&-\sum_{i=1}^m\frac{r_i}{2^n}\log\frac{r_i}{2^n} \\
  &=&-\sum_{i\in I_1}\frac{r_i}{2^n}\log\frac{r_i}{2^n}-\sum_{i\in I_2}\frac{r_i}{2^n}\log\frac{r_i}{2^n} \\
  &=&-\sum_{i\in I_1}\frac{s_{i_j}}{2^n}\log\frac{s_{i_j}}{2^n}-\sum_{i\in I_2}\frac{\sum_j s_{i_j}}{2^n}\log\frac{\sum_j s_{i_j}}{2^n} \\
  &=& -\sum_{i\in I_1}\frac{s_{i_j}}{2^n}\log\frac{s_{i_j}}{2^n}-\frac{1}{2^n}\sum_{i\in I_2}\left(\sum_j s_{i_j}\right)\left[\log\left(\sum_j s_{i_j}\right)-n\right] \\
  &=& -\sum_{i\in I_1}\frac{s_{i_j}}{2^n}\log\frac{s_{i_j}}{2^n}-\frac{1}{2^n}\sum_{i\in I_2}\left[\log\left(\sum_j s_{i_j}\right)^{\left(\sum_j s_{i_j}\right)}-n\left(\sum_j s_{i_j}\right)\right]\\
  &<&-\sum_{i\in I_1}\frac{s_{i_j}}{2^n}\log\frac{s_{i_j}}{2^n}-\frac{1}{2^n}\sum_{i\in I_2}\left[\log\left(\prod_j s_{i_j}^{s_{i_j}}\right)-n\left(\sum_j s_{i_j}\right)\right]\\
  &=&-\sum_{i\in I_1}\frac{s_{i_j}}{2^n}\log\frac{s_{i_j}}{2^n}-\frac{1}{2^n}\sum_{i\in I_2}\left[\sum_js_{i_j}\log s_{i_j}-n\left(\sum_j s_{i_j}\right)\right]\\
  &=&-\sum_{i\in I_1}\frac{s_{i_j}}{2^n}\log\frac{s_{i_j}}{2^n}-\sum_{i\in I_2}\sum_j\frac{s_{i_j}}{2^n}\log\frac{s_{i_j}}{2^n}\\
  &=&\HH(\sigma),
\end{eqnarray*}
namely, $\HH(\sigma)>\HH(\pi)$. Therefore, the clause (1) holds.

(2) It follows from (1) that $\HH$ reaches the maximum value when $\pi=\hat{\pi}$. In this case, we get by definition that
$$\HH(\hat{\pi})=-\sum_{i=1}^{2^n}\frac{1}{2^n}\log\frac{1}{2^n}=n.$$ This proves (2).

(3) By (1), we see that $\HH$ reaches the minimum value when $\pi=\check{\pi}$. In this case, the empty subset $\emptyset$ of $U$ has the rough set approximation $(\emptyset,\emptyset)$ and $U$ itself has the rough set approximation $(U,U)$. For any proper subset of $U$, if any, it has the rough set approximation $(\emptyset,U)$. Hence, $r_1=r_2=1$ and $r_3=2^n-2$. We thus obtain by definition that
\begin{eqnarray*}
\HH(\check{\pi}) &=&-\frac{1}{2^n}\log\frac{1}{2^n}-\frac{1}{2^n}\log\frac{1}{2^n}-\frac{2^n-2}{2^n}\log\frac{2^n-2}{2^n} \\
   &=&n-\frac{2^n-2}{2^n}\log(2^n-2).
\end{eqnarray*} Whence, (3) holds, finishing the proof of the proposition.
\end{proof}

Note that in the clause (3) of Theorem \ref{Pentfunc}, if $n=1$, the value of the corresponding summand $0\log0$ is taken to be $0$, which is consistent with the limit: $$\lim_{x\rightarrow0+}x\log x=0.$$

For later need, let us recall the following definition from \cite{Zhu09improved2}.
\begin{Def}\label{Dhomo}
Let $\langle U,\pi\rangle$ and $\langle V,\sigma\rangle$ be two
approximation spaces, and suppose that $f: U\longrightarrow V$ is a
mapping.
\begin{enumerate}[(1)]
  \item The mapping $f$ is called a {\rm homomorphism} from $\langle U,\pi\rangle$ to
  $\langle V,\sigma\rangle$ if for any $C\in\pi$,
there exists $D\in\sigma$ such that $f(C)\subseteq D$, where $f(C)=\{f(u)\,|\:u\in C\}$.
  \item A homomorphism $f$ is called a {\rm monomorphism} if
 $f$ is an injective mapping.
\item A monomorphism $f$ is called {\rm strictly monomorphic} if either
there exist $C\in\pi$ and $D\in\sigma$ such that $f(C)\subsetneq D$,
namely, $f(C)\subseteq D$ and $f(C)\neq D$, or $|V|>|U|$.
  \item The mapping $f$ is called an {\rm isomorphism} if the mapping $f: U\longrightarrow V$ is
bijective, and moreover, both $f$ and its inverse mapping $f^{-1}$
 are homomorphisms.
\end{enumerate}
\end{Def}

We can now state the following facts.
\begin{Prop}\label{Phomoentropy}
Let $\langle U,\pi\rangle$ and $\langle V,\sigma\rangle$ be two
approximation spaces with $|U|=|V|$, and let $f: U\longrightarrow V$ be a
mapping.
\begin{enumerate}[(1)]
  \item If $f$ is a monomorphism from $\langle U,\pi\rangle$ to $\langle V,\sigma\rangle$, in particular, $\pi\preceq\sigma$, then $\HH(\pi)\geq\HH(\sigma)$.
\item If $f$ is a strict monomorphism from $\langle U,\pi\rangle$ to $\langle V,\sigma\rangle$, in particular, $\pi\prec\sigma$, then $\HH(\pi)>\HH(\sigma)$.
  \item If $f$ is an isomorphism from $\langle U,\pi\rangle$ to $\langle V,\sigma\rangle$, then $\HH(\pi)=\HH(\sigma)$.
\end{enumerate}
\end{Prop}
\begin{proof}
It follows immediately from Definition \ref{DShannonent} and Theorem \ref{Pentfunc}.
\end{proof}

To measure the granularity with respect to the approximation operators $\ul{app}$ and $\ol{app}$ carried by the partition $\pi$, we introduce the concept of co-entropy, which corresponds to the information entropy in Definition \ref{DShannonent}.
\begin{Def}\label{Dnovelcoentropy}
Keep the notations as in Definition \ref{DShannonent}. The {\rm co-entropy} $\G(\pi)$ of $\langle U,\pi\rangle$ (with respect to the approximation operators $\ul{app}$ and $\ol{app}$) is defined by
\begin{equation}\label{Enovelcoentropy}
\G(\pi)=\G(P(\pi))=\sum_{i=1}^m\frac{r_i}{2^n}\log {r_i}.
\end{equation}
\end{Def}

The quantity $\G(\pi)$ furnishes a measure of the average granularity carried by
the partition $\pi$ as a whole. It follows immediately from definition that
\begin{equation}\label{EqHHG}
\HH(\pi)+\G(\pi)=n.
\end{equation}
It means that the two measures complement each other with respect to the
constant quantity $n=|U|$, which is invariant with respect to the choice of the
partition $\pi$ of $U$.

The co-entropy function $\G$ is of the following properties.
\begin{Thm}\label{Pcoentfunc}\
\begin{enumerate}[(1)]
  \item For any $\pi,\sigma\in\Pi(U)$, if $\sigma\prec\pi$, then $\G(\sigma)<\G(\pi)$.
  \item The co-entropy function $\G$ reaches the minimum value $0$ for the
finest partition $\hat{\pi}$.
  \item The co-entropy function $\G$ reaches the maximum value $\frac{2^n-2}{2^n}\log(2^n-2)$ for the coarsest partition
$\check{\pi}$.
\end{enumerate}
\end{Thm}
\begin{proof}
All the clauses follow directly from Theorem \ref{Pentfunc} and Eq. (\ref{EqHHG}).
\end{proof}

Similar to Proposition \ref{Phomoentropy}, we have the following observation.
\begin{Prop}\label{Phomocoentropy}
Let $\langle U,\pi\rangle$ and $\langle U,\sigma\rangle$ be two
approximation spaces with $|U|=|V|$, and let $f: U\longrightarrow V$ be a
mapping.
\begin{enumerate}[(1)]
  \item If $f$ is a monomorphism from $\langle U,\pi\rangle$ to $\langle V,\sigma\rangle$, in particular, $\pi\preceq\sigma$, then $\G(\pi)\leq\G(\sigma)$.
\item If $f$ is a strict monomorphism from $\langle U,\pi\rangle$ to $\langle V,\sigma\rangle$, in particular, $\pi\prec\sigma$, then $\G(\pi)<\G(\sigma)$.
  \item If $f$ is an isomorphism from $\langle U,\pi\rangle$ to $\langle V,\sigma\rangle$, then $\G(\pi)=\G(\sigma)$.
\end{enumerate}
\end{Prop}
\begin{proof}
It follows immediately from Proposition \ref{Phomoentropy} and Eq. (\ref{EqHHG}).
\end{proof}

As a corollary of Theorem \ref{Phomocoentropy} and Proposition \ref{Pcoentfunc}, we see that $\G$ is a partition measure on $U$ in the sense of \cite[Definition 3.4]{Zhu09improved2}, that is, $\G$ is nonnegative and satisfies the following two conditions: $\G(\sigma)<\G(\pi)$ if $\sigma\prec\pi$; $\G(\pi)=\G(\sigma)$ if there exists an isomorphism from $\langle U,\pi\rangle$ to $\langle V,\sigma\rangle$.

Note that our information entropy and co-entropy are not directly based on the blocks of a partition. Therefore, in general they do not satisfy the definition of expected granularity proposed in \cite{YZ09p}.

\section{Relationship of co-entropies between different universes}
In the last section, we have seen that if $f$ is a strict monomorphism from $\langle U,\pi\rangle$ to $\langle U,\sigma\rangle$, in particular, $\pi\prec\sigma$, then $\HH(\pi)>\HH(\sigma)$ and $\G(\pi)<\G(\sigma)$. In this section, we consider the monotonicities of $\HH$ and $\G$ for different universes. In other words, we compare $\HH(\pi)$ with $\HH(\sigma)$ and $\G(\pi)$ with $\G(\sigma)$ when there exists a strict monomorphism from $\langle U,\pi\rangle$ to $\langle V,\sigma\rangle$, where $|V|>|U|$. For convenience, we write $\langle U,\pi\rangle\hookrightarrow\langle V,\sigma\rangle$ if $|V|>|U|$ and there exists a strict monomorphism from $\langle U,\pi\rangle$ to $\langle V,\sigma\rangle$.

We start with the following observation on the entropy function $H$ and the co-entropy function $G$ reviewed in Section 2.2. Consider $\langle U_1,\pi_1\rangle=\langle\{1\},\{\{1\}\}\rangle$, $\langle U_2,\pi_2\rangle=\langle\{1,2\},\{\{1\},\{2\}\}\rangle$, and $\langle U_3,\pi_3\rangle=\langle\{1,2,3\},\{\{1,3\},\{2\}\}\rangle$. Clearly, $$\la U_1,\pi_1\ra\ha\la U_2,\pi_2\ra\ha\la U_3,\pi_3\ra.$$ It is easy to check by Definition \ref{DgranmeaG} that $H(\pi_1)=0$, $H(\pi_2)=1$, and $H(\pi_3)=\log3-\frac{2}{3}<1$. This means that the entropy function $H$ is not monotonic. By the way, we can get by a direct computation that $\G(\pi_1)=0$, $\G(\pi_2)=0$, and $\G(\pi_3)=\frac{1}{2}$.

Let us continue to discuss the monotonicity of co-entropy function $G$. Consider $\langle U_1,\pi_1\rangle=\langle\{1\},\{\{1\}\}\rangle$, $\langle U_2,\pi_2\rangle=\langle\{1,2\},\{\{1,2\}\}\rangle$, and $\langle U_3,\pi_3\rangle=\langle\{1,2,3\},\{\{1,2\},\{3\}\}\rangle$. Again, we see that $$\la U_1,\pi_1\ra\ha\la U_2,\pi_2\ra\ha\la U_3,\pi_3\ra.$$ It is easy to check by Definition \ref{DgranmeaRG} that $G(\pi_1)=0$, $G(\pi_2)=1$, and $G(\pi_3)=\frac{2}{3}$. This shows that the co-entropy function $G$ is not monotonic either. In this case, we can obtain by a direct computation that $\G(\pi_1)=0$, $\G(\pi_2)=\frac{1}{2}$, and $\G(\pi_3)=\frac{1}{2}$.

Finally, we address the monotonicity of entropy function $\HH$. Consider $\langle U_1,\pi_1\rangle=\langle\{1,2\},\{\{1,2\}\}\rangle$, $\langle U_2,\pi_2\rangle=\langle\{1,2,3\},\{\{1,2\},\{3\}\}\rangle$, and $\langle U_3,\pi_3\rangle=\langle\{1,2,3,4\},\{\{1,2,4\},\{3\}\}\rangle$. Obviously, we have that $$\la U_1,\pi_1\ra\ha\la U_2,\pi_2\ra\ha\la U_3,\pi_3\ra.$$ By a routine computation we can get that $\HH(\pi_1)=\frac{3}{2}$, $\HH(\pi_2)=\frac{5}{2}$, and $\HH(\pi_3)=\frac{13}{4}-\frac{3}{4}\log3<\frac{5}{2}$. Consequently, the entropy function $\HH$ is not monotonic either. On the other hand, it follows from Eq. (\ref{EqHHG}) that $\G(\pi_1)=\frac{1}{2}$, $\G(\pi_2)=\frac{1}{2}$, and $\G(\pi_3)=\frac{3}{4}+\frac{3}{4}\log3$.

As a result, in all the above three cases we always have that $$\G(\pi_1)\leq\G(\pi_2)\leq\G(\pi_3).$$ We thus conjecture that  $\G(\pi)\leq\G(\sigma)$ whenever $\langle U,\pi\rangle\hookrightarrow\langle V,\sigma\rangle$. Indeed, it holds true, as we will see later.

To prove the conjecture, it is convenient to introduce the following notion and a key lemma.

\begin{Def}\label{Doneext}
Let $\langle U,\pi\rangle$ be an approximation space and $a\not\in U$. The approximation space $\langle U\cup\{a\},\pi\cup\{\{a\}\}\rangle$ is called the {\rm one-point extension} of $\langle U,\pi\rangle$ by $a$. We say that $\langle V,\sigma\rangle$ is a one-point extension of $\langle U,\pi\rangle$ if $\langle V,\sigma\rangle=\langle U\cup\{a\},\pi\cup\{\{a\}\}\rangle$ for some $a$.
\end{Def}

For example, $\langle U_2,\pi_2\rangle=\langle\{1,2,3\},\{\{1,2\},\{3\}\}\rangle$ is the one-point extension of $\langle U_1,\pi_1\rangle=\langle\{1,2\},\{\{1,2\}\}\rangle$ by $3$.

The following lemma shows that one-point extension does not change co-entropy.
\begin{Lem}\label{Loneext}
Let $\langle V,\sigma\rangle$ be a one-point extension of $\langle U,\pi\rangle$. Then $\G(\sigma)=\G(\pi)$.
\end{Lem}
\begin{proof}
Suppose that $\pi=\{U_1,U_2,\ldots,U_k\}$ and $\sigma=\{U_1,U_2,\ldots,U_k,\{a\}\}$, where $a\not\in U$; assume that $\mathcal{A}(U,\pi)=\{(A_i,A'_i)\,|\,1\leq i\leq m\}$ and $\A_i=\left\{X\subseteq U\,\Bigl|\,\left(\ul{app}_\pi X, \ol{app}_\pi X\right)=(A_i,A'_i)\right\}$ with $|\A_i|=r_i$.
It thus follows that $$\mathcal{A}(V,\sigma)=\mathcal{A}(U,\pi)\cup\left\{\left(A_i\cup\{a\},A'_i\cup\{a\}\right)\,|\,1\leq i\leq m\right\}.$$For any $(B_i,B'_i)\in\mathcal{A}(V,\sigma)$, we write $\mathcal{B}_i$ for $\left\{X\subseteq V\,\Bigl|\,\left(\ul{app}_\sigma X, \ol{app}_\sigma X\right)=(B_i,B'_i)\right\}$ and $s_i$ for $|\mathcal{B}_i|$. If $(B_i,B'_i)=(A_i,A'_i)\in\mathcal{A}(U,\pi)$, then we see that $\mathcal{B}_i=\A_i$ and thus $s_i=r_i$ in this case. If $(B_i,B'_i)=\left(A_i\cup\{a\},A'_i\cup\{a\}\right)\in\left\{\left(A_i\cup\{a\},A'_i\cup\{a\}\right)\,|\,1\leq i\leq m\right\}$, then we have that $\mathcal{B}_i=\{X\cup\{a\}\,|\,X\in\A_i\}$ and $s_i=r_i$ still holds in this case. Therefore, we get by Definition \ref{Dnovelcoentropy} that
\begin{eqnarray*}
\G(\sigma)&=&\sum_{i=1}^m\frac{r_i}{2^{n+1}}\log {r_i}+\sum_{i=1}^m\frac{r_i}{2^{n+1}}\log {r_i}\\
  &=&\sum_{i=1}^m\frac{r_i}{2^{n}}\log {r_i} \\
  &=&\G(\pi),
\end{eqnarray*}finishing the proof of the lemma.
\end{proof}

For subsequent need, we would like to generalize Definition \ref{Doneext} as follows.
\begin{Def}\label{Dmoneext}
Let $\langle U,\pi\rangle$ and $\langle V,\sigma\rangle$ be two approximation spaces. We say that $\langle V,\sigma\rangle$ is a {\rm multi-one-point extension} of $\langle U,\pi\rangle$ if there are approximation spaces $\langle U_i,\pi_i\rangle$, $0\leq i\leq l$, with $\langle U_0,\pi_0\rangle=\langle U,\pi\rangle$ and $\langle U_l,\pi_l\rangle=\langle V,\sigma\rangle$ such that each $\langle U_i,\pi_i\rangle$, $1\leq i\leq l$, is a one-point extension of $\langle U_{i-1},\pi_{i-1}\rangle$.
\end{Def}

For example, $\langle V,\sigma\rangle=\langle \{1,2,3,4\},\{\{1,2\},\{3\},\{4\}\}\rangle$ is a multi-one-point extension of $\langle U,\pi\rangle=\langle \{1,2\},\{\{1,2\}\}\rangle$. In fact, we may take $\langle U_0,\pi_0\rangle=\langle U,\pi\rangle$, $\langle U_1,\pi_1\rangle=\langle \{1,2,3\},\{\{1,2\},\{3\}\}\rangle$, and $\langle U_2,\pi_2\rangle=\langle V,\sigma\rangle$.

The following fact follows immediately from Lemma \ref{Loneext}.
\begin{Coro}\label{Cmulti}
If $\langle V,\sigma\rangle$ is a multi-one-point extension of $\langle U,\pi\rangle$, then $\G(\sigma)=\G(\pi)$.
\end{Coro}

In light of the above corollary, let us refer to multi-one-point extensions as one-point extensions for simplicity. Further, we have the following observation.
\begin{Thm}\label{Tmultiext}Suppose that there is a monomorphism from $\langle U,\pi\rangle$ to $\langle V,\sigma\rangle$. If there exists $\langle U',\pi'\rangle$ that satisfies the following two conditions:
\begin{enumerate}[(a)]
  \item either $\langle U',\pi'\rangle=\langle U,\pi\rangle$ or $\langle U',\pi'\rangle$ is a one-point extension of $\langle U,\pi\rangle$,
  \item $\langle U',\pi'\rangle$ is isomorphic to $\langle V,\sigma\rangle$,
\end{enumerate}
then $\G(\pi)=\G(\sigma)$; otherwise, $\G(\pi)<\G(\sigma)$.
\end{Thm}
\begin{proof}
We first consider the case that there exists $\langle U',\pi'\rangle$ that satisfies the conditions (a) and (b). In this case, if $\langle U',\pi'\rangle=\langle U,\pi\rangle$ and $\langle U',\pi'\rangle$ is isomorphic to $\langle V,\sigma\rangle$, then $|V|=|U|$ and we see by Proposition \ref{Phomocoentropy} that $\G(\pi)=\G(\sigma)$. If $\langle U',\pi'\rangle$ is a one-point extension of $\langle U,\pi\rangle$ and $\langle U',\pi'\rangle$ is isomorphic to $\langle V,\sigma\rangle$, then we get that $\G(\pi)=\G(\pi')$ by Corollary \ref{Cmulti} and $\G(\pi')=\G(\sigma)$ by Proposition \ref{Phomocoentropy}. Consequently,  $\G(\pi)=\G(\sigma)$.

We now consider the case that there does not exist $\langle U',\pi'\rangle$ such that the conditions are satisfied. It forces that the monomorphism, say $f$, from $\langle U,\pi\rangle$ to $\langle V,\sigma\rangle$ is strict. Two cases need to consider. One is that $|V|=|U|$. In this case, it follows from Proposition \ref{Phomocoentropy} that $\G(\pi)<\G(\sigma)$. The other case is that $|V|>|U|$. In this case, let us set $\langle V_1,\sigma_1\rangle=\langle f(U),f(\pi)\rangle$, where $f(U)$ is the image of $U$ under $f$ and $f(\pi)=\{f(U')\,|\,U'\in\pi\}$. In fact, $f$ gives rise to an isomorphism between $\langle U,\pi\rangle$ and $\langle V_1,\sigma_1\rangle$. Therefore, $\G(\pi)=\G(\sigma_1)$. Note that $V_1=f(U)\subseteq V$. We now take $\langle V_2,\sigma_2\rangle$ as follows:
$$V_2=V, \  \sigma_2=\sigma_1\cup\left\{\{a\}\,|\,a\in V\backslash V_1\right\}.$$
It follows that $\langle V_2,\sigma_2\rangle$ is a one-point extension of $\langle V_1,\sigma_1\rangle$. Hence, $\G(\sigma_1)=\G(\sigma_2)$. Because $f$ is a strict monomorphism, we see that $\langle U,\pi \rangle\hookrightarrow\langle V_2,\sigma_2\rangle$ and $\sigma_2\prec\sigma$. Whence, we get $\G(\sigma_2)<\G(\sigma)$ by Theorem \ref{Pcoentfunc}. As a result, $\G(\pi)<\G(\sigma)$. This completes the proof of the theorem.
\end{proof}

Let us provide an informal explanation of Theorem \ref{Tmultiext}. The hypothesis that there is a monomorphism from $\langle U,\pi\rangle$ to $\langle V,\sigma\rangle$ means that $\langle U,\pi\rangle$ is finer than $\langle V,\sigma\rangle$. In the special case that the monomorphism is not strict, we have that $\langle U,\pi\rangle$ and $\langle V,\sigma\rangle$ are  isomorphic, and thus, they have the same co-entropy. If the monomorphism is strict, then after renaming the elements of $U$, we can get a finer partition than $\langle V,\sigma\rangle$ by using one-point extensions. Theorem \ref{Tmultiext} says that $\G(\pi)<\G(\sigma)$ if $\pi$ is finer than $\sigma$.

We end this section with several examples.
\begin{Eg}
A trivial example is that $\langle U,\pi\rangle=\langle \{1,2,3\},\{\{1,2\},\{3\}\}\rangle$ and $\langle V,\sigma\rangle=\langle \{a,b,c\},\{\{a,b\},\{c\}\}\rangle$. The mapping $f$ that maps $1$, $2$, and $3$ to $a$, $b$, and $c$ respectively is a monomorphism. In fact, $f$ is an isomorphism. Hence, $\G(\pi)=\G(\sigma)$. A direct computation shows that $\G(\pi)=0.5=\G(\sigma)$.

Consider $\langle U,\pi\rangle=\langle \{1,2\},\{\{1,2\}\}\rangle$ and $\langle V,\sigma\rangle=\langle \{a,b,c,d\},\{\{a,b\},\{c\},\{d\}\}\rangle$. The mapping $f$ that maps $1$ and $2$ to $a$ and $b$ respectively is a monomorphism, which yields that $\langle U,\pi\rangle$ is isomorphic to $\langle V_1,\sigma_1\rangle=\langle \{a,b\},\{\{a,b\}\}\rangle$. Clearly, we can get $\langle V,\sigma\rangle$ by one-point extensions of $\langle V_1,\sigma_1\rangle$. Therefore, $\G(\pi)=\G(\sigma)$. On the other hand, we can get by a computation that $\G(\pi)=0.5=\G(\sigma)$.

Finally, consider $\langle U,\pi\rangle=\langle \{1,2\},\{\{1,2\}\}\rangle$ and $\langle V,\sigma\rangle=\langle \{a,b,c,d\},\{\{a,b\},\{c,d\}\}\rangle$. As mentioned earlier, the mapping $f$ that maps $1$ and $2$ to $a$ and $b$ respectively is a monomorphism, which gives an isomorphism between $\langle U,\pi\rangle$ and $\langle V_1,\sigma_1\rangle=\langle \{a,b\},\{\{a,b\}\}\rangle$. We can get $\langle V_2,\sigma_2\rangle=\langle \{a,b,c,d\},\{\{a,b\},\{c\},\{d\}\}\rangle$ by one-point extensions of $\langle V_1,\sigma_1\rangle$. Clearly, $\sigma_2\prec\sigma$. As a result, $\G(\pi)<\G(\sigma)$. On the other hand, we can obtain by a direct computation that $\G(\pi)=0.5$ and $\G(\sigma)=0.75$.
\end{Eg}

\section{Conclusion}
In this paper, we have proposed the novel notions of entropy and co-entropy by taking both partitions and the lower and upper approximations into account. Some desirable properties of the entropy and co-entropy have been presented. Furthermore, we have investigated the relationship of co-entropies between different universes.

There are several problems which are worth further studying. Firstly, the present work focuses on the classical rough sets based on partitions. It would be interesting to generalize the notions of entropy and co-entropy here into the framework of covering rough sets \cite{Bry89,Pom87,Zak83} or fuzzy rough sets \cite{DP90}. It is also interesting to compare the entropies (co-entropies) under some special homomorphisms such as neighborhood-consistent functions introduced in \cite{zhu2010some}. Secondly, it remains to develop the corresponding roughness measure based on the entropy or co-entropy for measuring numerically the
roughness of an approximation. Finally, the conditioned entropy and conditioned co-entropy \cite{bianucci2009information} are yet to be addressed in our
framework.

\section*{Acknowledgements}
This work was supported by the National Natural Science Foundation of China under Grants 60821001, 60873191, 60903152, and 61070251.


\end{document}